\RequirePackage{fix-cm}
\documentclass[smallcondensed]{svjour3}     
\smartqed  

\newcommand{\eqname}[1]{\tag*{#1}}
\usepackage[lofdepth,lotdepth]{subfig}
\usepackage{times,amsmath,epsfig}
\usepackage{multicol}
\usepackage{hyphenat}
\usepackage{colortbl}
\usepackage{amsfonts}
\usepackage{amssymb}
\usepackage{mathrsfs}
\usepackage[ruled,vlined,linesnumbered]{algorithm2e}
\usepackage{graphicx}
\usepackage{multirow}
\usepackage{float}
\usepackage{caption}
\usepackage{times} 
\usepackage{subscript}
\usepackage{helvet}  
\usepackage{courier}  
\usepackage{url}
\usepackage[utf8]{inputenc} 
\usepackage[T1]{fontenc} 
\usepackage{url}            
\usepackage{amsthm}
\usepackage{booktabs}       
\usepackage{amsfonts}       
\usepackage{nicefrac}       
\usepackage{microtype}      
\usepackage{amsmath}
\usepackage{enumitem}
\usepackage{comment}
\begin{document}

\title{Improved Linear Embeddings via Lagrange Duality
}


\author{Kshiteej Sheth$^{1}$         \and
        Dinesh Garg$^{1}$ \and \\
		Anirban Dasgupta$^{1}$  \\ \{kshiteej.sheth,dgarg,anirbandg\}@iitgn.ac.in
}


\institute{Corresponding Author : Kshiteej Sheth 
             \\
              Tel.: +91-9825607148
          \\
              \email{kshiteej.sheth@iitgn.ac.in}           
\\
\\
$^1$Indian Institute of Technology, Palaj, Gandhinagar-382355, Gujarat, India. 
}
\date{Received: date / Accepted: date}

\maketitle

\begin{abstract}
Near isometric orthogonal embeddings to lower dimensions are a fundamental tool in data science and machine learning.
In this paper, we present the construction of such embeddings that minimizes the maximum distortion for a given set of points. 
We formulate the problem as a non convex constrained optimization problem. We first construct a primal relaxation and then use the theory
of Lagrange duality to create dual relaxation. We also suggest a polynomial time algorithm based on the theory of convex optimization to solve the dual relaxation provably. We provide a theoretical upper bound on the approximation guarantees for our algorithm, which depends only on the spectral properties of the dataset. We experimentally demonstrate the superiority of our algorithm compared to baselines in terms of the scalability and the ability to achieve lower distortion.\keywords{Near isometric embeddings \and Dimensionality reduction \and Convex and Non-Convex optimization \and Convex relaxations}
\end{abstract}
\newpage
\section{Introduction}
One of the fundamental tasks in data science, machine learning, and signal processing applications involving high dimensional 
data is to embed the data into lower dimensional spaces while preserving pairwise distances (aka similarities) between the data points. 
Applications include clustering~\cite{Badoiu02}, neighborhood preserving projections, hashing~\cite{Indyk98}, etc.  
Such embeddings are called as {\em Near isometric embeddings}. Formally, given a set of data points 
$\{\boldsymbol{u_i}\}_{i=1}^r$ where $\boldsymbol{u_i} \in \mathbb{R}^d$ $\forall$ $i \in [r]$, our goal is to find a function 
$f:\mathbb{R}^d\mapsto \mathbb{R}^k$, where $k \ll d$, which satisfies the following inequality:\\
\begin{equation}
1- \epsilon \leq \frac{\|f(\boldsymbol{u_i}) - f(\boldsymbol{u_j}) \|_2^2}{\|\boldsymbol{u_i} - \boldsymbol{u_j}\|_2^2} \leq 1 +  \epsilon \hspace{0.3cm} \forall i,j \in [r] \label{JL_Inequality}
\end{equation}
for a small enough $\epsilon$. Intuitively, it is clear that there is a trade-off between projected dimension ($k$) and maximum distortion ($\epsilon$). 

A widely celebrated result of \cite{Johnson84} says that for a given set of points, an appropriately scaled random linear transformation from $\mathbb{R}^d$ to $\mathbb{R}^k$ can achieve a distortion of $\epsilon$ for $k = \mathcal{O}(\log(r)/\epsilon^2)$, with high probability. 
Such {\em data oblivious} linear embeddings are popularly known as {\em JL embeddings}. 
Further, it was shown by \cite{Noga03,jayram2013optimal} that such bounds are tight. 
That means, there exists a set of $r$ points that necessarily require $\Omega(\log(r)/\epsilon^2)$ dimensions in order to be embedded with distortion at most $\epsilon$.  This finding leaves an open question of whether one can project the data into a further lower dimensional space by exploiting the geometry of the dataset,  while having distortion no more than $\epsilon$.  
 Principal Component Analysis (PCA), while being
the de-facto data dependent method for constructing low dimensional representations,
minimizes only the average distortion over the entire set of points. Individual data points can still have an arbitrary distortion under PCA. 


Motivated by these observations, our goal, in this paper, is to address the question of how one can construct a data-dependent orthogonal linear embedding with minimal distortion. An orthonormal linear embedding corresponds to orthogonal projection of $\mathbb{R}^d$ onto a $k$-dimensional subspace. When the distortion is under some specified threshold, we call such an embedding as {\em near isometric orthogonal linear embedding}. One immediate consequence of 
the orthogonal projection is that the upper bound of inequality (\ref{JL_Inequality}) becomes trivially $1$ (an orthogonal projection can never increase the length of a vector). Thus, we only need to make sure that the lower bound is as close to $1$ as possible. 

\cite{Indyk13} formulated the same problem as a non-convex optimization problem and suggested a semidefinite programming (SDP) based relaxation. Further, they proposed two rounding schemes for the SDP solution so that it yield an approximate solution of the original problem. Their method does have a provable guarantee (either a $\mathcal{O}(k+1)$ approximation, or a bi-criteria one) but is computationally quite expensive as observed in our experimentations. This is potentially due to the large number of constraints in the corresponding SDP relaxation. In light of this, our contributions in this paper can be summarized as follows.
\begin{enumerate}
\item We take an alternative approach to solve the problem of near isometric orthogonal linear embedding, which is a non-convex optimization problem (referred to as the {\em primal} problem). In section 2, we first develop a relaxed version of this primal problem (referred to as the {\em relaxed primal} problem) and then construct its Lagrangian dual (referred to as {\em relaxed dual} problem). Our relaxed primal problem still remains a non-convex problem but our relaxed dual becomes a convex problem. We solve the relaxed dual problem by projected gradient method and find a dual optimal solution.
\item In section 2.2, we show that the solution of the relaxed primal problem corresponding to the relaxed dual optimal solution is indeed a feasible and an approximate solution of the original primal problem. 
\item In section 2.3, we prove a theoretical upper bound on the ratio of distortion achieved by our algorithm to  the optimal distortion. This bound depends on the spectral properties of the given dataset. Using this theoretical bound, we argue that for the {\em well behaved} dataset, our distortion always remains within 2 times the optimal distortion. Also, our distortion starts approaching close to the optimal distortion as the rank of the dataset starts increasing from 2. 
\item Our relaxed dual problem has much less constraints (much less than the primal), which makes solving the dual relaxation computationally efficient. 
\item We experimentally verify that our method not only takes lower time but also achieves lower distortion compared to almost all the standard baselines including the recent one proposed by \cite{Indyk13}.
\end{enumerate}
A few variants of near-isometric linear embeddings have also been studied. For example, see \cite{Jerryluo} and \cite{Adagio}. In such cases, however, the linear transformation is ``near-isometric''  only (and not necessarily orthogonal) on the input data points, and does not come with any guarantee about the stretch (lower or upper bound) for out of sample data points.  

\section{Duality for Linear Embeddings}
Let $\{\boldsymbol{u_i}\}_{i=1}^r$ be a given dataset where $\boldsymbol{u_i} \in \mathbb{R}^d$ $\forall$ $i \in [r]$. As mentioned in the previous section, our goal is to orthogonally project (aka embed) these data into a subspace of dimension $k$. Because the projection is linear as well as orthogonal, the required condition for near isometric embedding becomes
\[1- \epsilon \leq \frac{\|\mathbf{P}(\boldsymbol{u_i} - \boldsymbol{u_j}) \|_2^2}{\|\boldsymbol{u_i} - \boldsymbol{u_j}\|_2^2}\;\;\;\forall i < j \in[r] \] 
where, $\mathbf{P}$ is the required linear map. 
In light of this inequality, we do the following - compute normalized pairwise differences given by $\left(\boldsymbol{u_i} - \boldsymbol{u_j}\right)/\lVert\left(\boldsymbol{u_i} - \boldsymbol{u_j}\right)\rVert$ for every pair $(i,j)$, where $i <j; (i,j) \in[r]$, and then orthogonally project them to preserve their squared lengths. This way, we get $r \choose 2$ vectors each being unit length. If we let $n = {r \choose 2}$ then, we can denote such vectors by $\boldsymbol{x}_1, \boldsymbol{x}_2, \ldots, \boldsymbol{x}_n$ and our goal becomes to orthogonally project them so as to preserve their squared lengths.
For this, we let $\boldsymbol{v}_1\neq0, \boldsymbol{v}_2\neq0,\ldots, \boldsymbol{v}_k\neq0$ be some orthogonal basis vectors of a $k$-dimensional subspace ${\mathcal{S}}_k$ of $\mathbb{R}^d$.  Further, let $\mathbf{V}$ be a $d \times k$ matrix whose columns are given by the orthogonal basis vectors $\boldsymbol{v}_1, \boldsymbol{v}_2,\ldots, \boldsymbol{v}_k$. Let  $\boldsymbol{\tilde{x}}_i \in {\mathcal{S}}_k$ be the vector obtained by projecting $\boldsymbol{x}_i$ orthogonally into the space ${\mathcal{S}}_k$. Then, the problem of pushing the lower bound in inequality (\ref{JL_Inequality}) as close to $1$ as possible is equivalent to minimizing a real number $\text{\LARGE{$\epsilon$}}$ such that the distortion of each of the $x_i$'s after projection is less than $\text{\LARGE{$\epsilon$}}$. Here, the distortion is given by the squared length of the residual vector $(\boldsymbol{{x}}_i-\boldsymbol{\tilde{x}}_i)$, which would be equal to $1-\lVert \boldsymbol{\tilde{x}}_i\rVert^2$. Formally, it can be stated as the following optimization problem which we refer to as the \textbf{Primal} problem.
\rule{\columnwidth}{0.5pt}
 \begin{align}
 &\underset{\boldsymbol{v}_1\neq0,\ldots,\boldsymbol{v}_k \neq0, \text{\large{$\epsilon$}}}{\text{Minimize }}\text{\LARGE{$\epsilon$}}\nonumber\\
 & \text{subject to}\ \ 1- \lVert \mathbf{V}^{\top} \boldsymbol{x}_i \rVert_2^2\le \text{\LARGE{$\epsilon$}},\;\forall i =1,\ldots, n\nonumber\\
& \boldsymbol{v}_j^{\top}\boldsymbol{v}_j = 1,\;\forall j =1,\ldots, k\nonumber\\
& \boldsymbol{v}_j^{\top}\boldsymbol{v}_m = 0,\;\forall j  \neq m \eqname{(Primal)} \label{Primal}
 \end{align} 
\rule{\columnwidth}{0.5pt}
The last two constraints force the matrix $\mathbf{V}$ to be orthonormal. The first constraint, on the other hand, identifies the vector for whom $1- \lVert \boldsymbol{\tilde{x}}_i \rVert^2$ is maximum. 
 The objective function tries to minimize  $1- \lVert \boldsymbol{\tilde{x}}_i \rVert^2$ for such a vector. 
 Observe, whenever $\mathbf{V}$ is an orthogonal matrix, we have $\lVert \boldsymbol{\tilde{x}}_i \rVert^2 = \lVert \mathbf{V}^{\top} \boldsymbol{x}_i \rVert^2 \le 1$ for each $i=1,\ldots,n$. 

Let $p^*$ be the optimal value of the above problem. It is straightforward  to verify that $0 \le p^* \le 1$. 
\subsection{Lagrangian Dual}
The \ref{Primal} problem is a non-convex optimization problem as the feasible region forms a non-convex set. So we cannot hope to exactly solve this problem efficiently and hence we aim for approximate solutions.

For this, we momentarily ignore the the last equality constraint in \ref{Primal} problem, which enforces orthogonality of the $v_j$'s, and consider the following {\em Relaxed Primal} problem. We will later prove that inspite of this relaxation, we will reach a solution in which the vectors are indeed orthogonal, thereby satisfying the constraints of the \ref{Primal} problem and justifying the effectiveness of our approach.\\
\rule{\columnwidth}{0.5pt}
\vspace{-0.5cm}
\begin{align*}
&\underset{\boldsymbol{v}_1\neq0,\ldots,\boldsymbol{v}_k\neq 0 , \text{$\epsilon$}}{\text{Minimize }} \text{\LARGE{$\epsilon$}}\\
& \text{subject to} \ \ 1- \lVert \mathbf{V}^{\top} \boldsymbol{x}_i \rVert_2^2\le \text{\LARGE{$\epsilon$}},\;\forall i =1,\ldots, n\nonumber\\
& \boldsymbol{v}_j^{\top}\boldsymbol{v}_j = 1,\;\forall j =1,\ldots, k \eqname{(Relaxed Primal)} \label{Relaxed_Primal}
\end{align*} 
\rule{\columnwidth}{0.5pt}
The following lemma is going to be useful in subsequent developments, we omit the proof as it is fairly straightforward.
\begin{lemma} \label{lemma1}Let $\widehat{p^*}$ is the optimal value of the \ref{Relaxed_Primal} problem. Then, followings hold true.
\begin{enumerate}[label=(\Alph*)]
\item $\widehat{p^*} \le p^* \le 1$.
\item If an optimal solution of the \ref{Relaxed_Primal} problem satisfies the constraints of the \ref{Primal} problem then that solution must be an optimal solution for the  \ref{Primal} problem also.
\end{enumerate}
\end{lemma}
Note, the constraint $\boldsymbol{v}_j^{\top}\boldsymbol{v}_j = 1$ forces both \ref{Primal} as well as \ref{Relaxed_Primal} problems to be non-convex. The reason being following - {\em any norm is a convex function and the level set of a norm is a non-convex set}.  This motivates us to work with the Lagrangian dual of the \ref{Relaxed_Primal} problem and develop a strategy to get an approximate solution of the original \ref{Primal} problem.
  \subsection{Dual of Relaxed Primal} 
The Lagrangian dual for \ref{Relaxed_Primal} will always be a convex program irrespective of convexity of the \ref{Relaxed_Primal}. For this, we write down the Lagrangian function of \ref{Relaxed_Primal} as follows:
\begin{align}
& L(\boldsymbol{v}_1,\ldots,\boldsymbol{v}_k,\;\text{\LARGE{$\epsilon$}},\; \lambda_1,\ldots,\lambda_n,\mu_1,\ldots,\mu_k) \nonumber \\
& = \text{\LARGE{$\epsilon$}} + \sum_{i=1}^{n} \lambda_i \left(1- \lVert \mathbf{V}^{\top} \boldsymbol{x}_i \rVert_2^2 - \text{\LARGE{$\epsilon$}}\right) 
+ \sum_{j=1}^{k} \mu_j \left(\boldsymbol{v}_j^{\top} \boldsymbol{v}_j - 1\right)  
\end{align}
Observe, $\boldsymbol{v}_1,\ldots,\boldsymbol{v}_k,\;\text{\LARGE{$\epsilon$}}$ are the primal variables and
$\lambda_1,\ldots,\lambda_n,\mu_1,\ldots,\mu_k$ are the dual variables. The dual function is given by
\begin{align*}
g(\boldsymbol{\lambda},\boldsymbol{\mu})=\underset{\boldsymbol{v}_1\neq,\ldots,\boldsymbol{v}_k\neq0,\text{\large{$\epsilon$}}}{\text{argmin}} L(\boldsymbol{v}_1,\ldots ,\boldsymbol{v}_k,\text{\LARGE{$\epsilon$}},\lambda_1,\ldots \lambda_n,\mu_1,\ldots,\mu_k) 
\end{align*}
For any set of values $\{\lambda_i, \lambda_i \ge 0\}$, we define matrix 
$M=\sum_{i=1}^{n} \lambda_i \boldsymbol{x}_i\boldsymbol{x}_i^{\top}=\mathbf{X}^{\top}diag(\lambda_1,\ldots,\lambda_n) \mathbf{X}$.
The following lemma characterizes the dual solution in terms of this matrix $M$. 

\begin{lemma} \label{dual_lemma} Define $M=\sum_{i=1}^{n} \lambda_i \boldsymbol{x}_i\boldsymbol{x}_i^{\top}=\mathbf{X}^{\top}diag(\lambda_1,\ldots,\lambda_n) \mathbf{X}$. Then for any given values of dual variables $(\lambda_1,\ldots,\lambda_n,\mu_1,\ldots,\mu_k)$, $\lambda_i \ge 0$, we have
 \begin{align*}
 g(\boldsymbol{\lambda},\boldsymbol{\mu}) = \begin{cases}
                                              1-\sum\nolimits_{j=1}^{k} \mu_j :\text{If }\sum_{i=1}^{n}\lambda_i=1 \text{ and } \{\mu_j\}_{j=1}^{k}\\
                                              \ \text{ are the} \text{ top-$k$ eigenvalues of } M\\
                                              -\infty \text{ or  undefined}: \text{Otherwise}
                                            \end{cases}
 \end{align*}   
 Furthermore,the top-$k$ eigenvectors of the matrix $M$ are the minimizers of Lagrangian function. 
 \end{lemma}
\begin{proof}
In order to get the dual function, we set the gradient of Lagrangian function (with respect to primal variables) to be zero. This gives us the following conditions.
\begin{eqnarray}
{\partial L}/{\partial \boldsymbol{v}_j} &=& 2 \mu_j\boldsymbol{v}_j - 2 \left(\sum\nolimits_{i=1}^{n} \lambda_i \boldsymbol{x}_i\boldsymbol{x}_i^{\top}\right)\boldsymbol{v}_j=0 \label{Lagrangian_Partial1}\\
{\partial L}/{\partial \text{\LARGE{$\epsilon$}}} &=& 1-\sum\nolimits_{i=1}^{n} \lambda_i =0 \label{Lagrangian_Partial2}
\end{eqnarray}
Using the definition of matrix $M$, we first rewrite the expression for Lagrangian as follows.
 \begin{align}
 & L= \text{\LARGE{$\epsilon$}} \left(1-\sum\nolimits_{i=1}^{n} \lambda_i\right) + \sum\nolimits_{i=1}^{n} \lambda_i -\sum\nolimits_{j=1}^{k}\mu_j \nonumber\\
 &  - \sum\nolimits_{j=1}^{k} \boldsymbol{v}_j^{\top} \left(M - \mu_j \mathbf{I}\right)\boldsymbol{v}_j 
 \end{align}
 The minimum value of the Lagrangian has to satisfy the first order conditions that we get by 
 setting the gradient of Lagrangian function (with respect to primal variables) to be zero. This gives us the following conditions.
\begin{align}
{\partial L}/{\partial \boldsymbol{v}_j} &= 2 \mu_j\boldsymbol{v}_j - 2 \left(\sum\nolimits_{i=1}^{n} \lambda_i \boldsymbol{x}_i\boldsymbol{x}_i^{\top}\right)\boldsymbol{v}_j=0 \label{Lagrangian_Partial1}\\
{\partial L}/{\partial \text{\LARGE{$\epsilon$}}} &= 1-\sum\nolimits_{i=1}^{n} \lambda_i =0 \label{Lagrangian_Partial2}
\end{align}
By looking at these conditions, the following claims follow.
 \begin{enumerate}
 \item From equation \eqref{Lagrangian_Partial2}, $\sum_{i=1}^n\lambda_i = 1$. An alternative argument could be as follows:
 since $\LARGE{\epsilon}$ is unconstrained, if $\sum_{i=1}^n\lambda_i \neq 1$, then it is possible to set $\LARGE{\epsilon}$ such that
  $g(\boldsymbol{\lambda},\boldsymbol{\mu}) = -\infty$.

 \item  Equation~\eqref{Lagrangian_Partial1} implies that for achieving minimum, the primal variables $\boldsymbol{v}_j $ must be an eigenvector of the 
 matrix $M=\sum_{i=1}^{n} \lambda_i \boldsymbol{x}_i\boldsymbol{x}_i^{\top}$ whose corresponding eigenvalue is
 the dual variable $\mu_j$. 
 
 We argue that setting $\mu_j$ to be any value other than the top-$j^{th}$ eigenvalue of the $M$
 leads to $g(\boldsymbol{\lambda},\boldsymbol{\mu})$ being either $-\infty$ or undefined. For this, note that under the condition given by \eqref{Lagrangian_Partial2}, Lagrangian becomes 
 \begin{eqnarray*}
 1-\sum\limits_{j=1}^{k} \mu_j - \sum\limits_{j=1}^{k} \boldsymbol{v}_j^{\top} \left(M - \mu_j \mathbf{I}\right)\boldsymbol{v}_j
 \end{eqnarray*}
Without loss of generality, we can assume that $\mu_1 \ge \mu_2 \ge \ldots \ge \mu_k$. 
 Suppose $\gamma_1\ge \gamma_2 \ge \ldots \ge \gamma_d$ are the eigenvalues of the  the matrix $M$. 
 For contradiction assume that $\mu_1$ is not an eigenvalue of the matrix $M$. Now, let us consider two different cases.
 \begin{enumerate}
 \item {\bf Case of [$\mathbf{\gamma_1  > \mu_1}$]:} In this case, we can assign $\boldsymbol{v}_1 $ to be scaled eigenvector corresponding to the eigenvalue $\gamma_1$ and that would drive the Lagrangian value towards $-\infty$. Therefore, in order to avoid the $-\infty$ value for the dual function, we must have $\mu_1 \ge \gamma_1$.
 \item {\bf Case of [$\mathbf{\mu_1  > \gamma_1}$]:}  Under this scenario, we again consider two subcases:
 \begin{enumerate}
 \item {\bf Subcase of [$\mathbf{\gamma_1  > \mu_k}$]:} In such a case, by the previous argument, we can again drive the Lagrangian to $-\infty$ by assigning $\boldsymbol{v}_k$ appropriately. 
 \item {\bf Subcase of [$\mathbf{\mu_k > \gamma_1}$]:} For this subcase, note that matrix $\left(M - \mu_j \mathbf{I}\right)$ would be a negative definite matrix for each $j \in[k]$ and 
 hence $\boldsymbol{v}_j=0,\;\forall j \in [k]$ will minimize the Lagrangian. However, all $v_j$ cannot simultaneously be zero as per the constraint.
 Hence, in this case, the minimum is not defined within the (open) set of $\{v_j ,\; \forall j\ v_j \neq 0 \}$. 
 
 \end{enumerate}
 \end{enumerate}
Thus, we can conclude that $\mu_1$ should be equal to $\gamma_1$ in order to make sure that the dual
function is well defined and has a finite value. 
This also means that $v_1$ has to be the top eigenvector of $M$. We can inductively apply the same argument for other $\mu_j$ also to get the desired claim.
\end{enumerate}
\end{proof} 
Lemma \ref{dual_lemma} suggests that  the vectors  $\boldsymbol{v}_1 \neq 0,\ldots, \boldsymbol{v}_k \neq 0$ which minimize Lagrangian must be the top-$k$ eigenvectors of the matrix $M$ and in such a case, the dual function can be given as follows $g(\boldsymbol{\lambda},\boldsymbol{\mu}) = 1- \sum\nolimits_{i=1}^{k} \mu_j$ where, $\mu_j$ is top $j^{th}$ eigenvalue of the matrix $M$. The dual optimization problem for the \ref{Relaxed_Primal}, thus becomes\\
\rule{\columnwidth}{0.5pt}
\vspace{-0.5cm}
\begin{align}
&\underset{\boldsymbol{\lambda},\boldsymbol{\mu}}{\text{Maximize}} \ \ g(\boldsymbol{\lambda},\boldsymbol{\mu}) =1- \sum\nolimits_{j=1}^{k} \mu_j\nonumber \\
&\ \text{subject to }\nonumber\\
& \mu_j \text{ is top $j^{th}$ eigenvalue of } M,\;\forall j=1,\ldots,k\nonumber\\
&\sum\nolimits_{i=1}^{n}\lambda_i = 1\nonumber\\
&\lambda_i \ge 0,\;\forall i =1,\ldots, n
\eqname{(Relaxed Dual)}\label{Dual_Relaxed_Primal}
\end{align} 
\rule{\columnwidth}{0.5pt}
Let $\widehat{d^*}$ be the optimal value of the \ref{Dual_Relaxed_Primal} problem. In what follows, we state a few key observations with regard to the above primal-dual formulation discussed so far. 
\begin{lemma}\label{dual_lemma1} For the given \ref{Primal}, \ref{Relaxed_Primal}, and \ref{Dual_Relaxed_Primal} programs, following hold true.
 \begin{enumerate}[label=(\Alph*)]
 \item $\widehat{d^*} \le \widehat{p^*} \le p^* \le 1$
 \item If $\left(\boldsymbol{\lambda}, \boldsymbol{\mu}\right)$ is a feasible solution of \ref{Dual_Relaxed_Primal}, then
 \begin{eqnarray*}
 0 & \le & \mu_j,\;\forall j\in [k]\\
 0 &\le& \sum\nolimits_{j=1}^{k} \mu_k \le Tr(M) =\sum\nolimits_{i=1}^n \lambda_i=1
 \end{eqnarray*}
 \item $\text{Rank}(M)\le \text{Rank}(\mathbf{X}^{\top}\mathbf{X})=\text{Rank}(\mathbf{X})$
 \item Let $(\boldsymbol{\lambda}^*,\boldsymbol{\mu}^*)$ be the point of optimality for \ref{Dual_Relaxed_Primal} problem. Then top-$k$ eigenvectors $\boldsymbol{v}_1^*,\ldots,\boldsymbol{v}_k^*$ of the matrix $M=\sum_{i=1}^{n}\lambda_i^*\boldsymbol{x}_i\boldsymbol{x}_i^{\top}=\boldsymbol{\lambda}^* \mathbf{X}^{\top}\mathbf{X}$ satisfy the following properties.
\begin{enumerate}
\item These vectors minimize the Lagrangian function $L(\boldsymbol{v}_1,\ldots,\boldsymbol{v}_k,\;\text{\LARGE{$\epsilon$}},\; \lambda_1^*,\ldots,\lambda_n^*,\mu_1^*,\ldots,\mu_k^*)$.
\item These vectors form a feasible solution for both \ref{Relaxed_Primal} and \ref{Primal} problems thus allowing us to use vectors obtained from the \ref{Dual_Relaxed_Primal} as approximate solutions of the \ref{Primal}.
\item Let $\text{\LARGE{$\epsilon$}}_{ALG}$ be the \ref{Relaxed_Primal} objective function value (which is also the value for the \ref{Primal} objective function) corresponding to the feasible solution  $\boldsymbol{v}_1^*,\ldots,\boldsymbol{v}_k^*$ then we must have $p^* \le \text{\LARGE{$\epsilon$}}_{ALG} \le 1$.
\end{enumerate}	   
\end{enumerate}
\end{lemma}
\begin{proof}
\hfill{}\\
{\bf Part (A):} 
The inequalities follow from weak duality theorem and the fact that $\widehat{p^*}$ is the optimal solution of the (Relaxed Primal) problem, whereas $p^*$ is the optimal solution of the (Primal) problem. The last inequality follows from Lemma \ref{lemma1}. \\\\
{\bf Part (B):} Let $(\boldsymbol{\lambda}, \boldsymbol{\mu})$ be a feasible solution for (Relaxed Dual) problem then by Lemma \ref{dual_lemma}, we can claim that $\mu_j \ge 0$ because $\mu_j$ must be an eigenvalue of the matrix $\mathbf{M}$ and this matrix is always a positive semidefinite. Further, notice that
\begin{align*}
& Tr(\mathbf{M}) = Tr(\sum_{i=1}^{n}\lambda_i\boldsymbol{x}_i\boldsymbol{x}_i^{\top})\\
& = \sum_{i=1}^{n}\lambda_i \; Tr(\boldsymbol{x}_i\boldsymbol{x}_i^{\top})=\sum_{i=1}^{n}\lambda_i Tr(\boldsymbol{x}_i^{\top}\boldsymbol{x}_i)= \sum_{i=1}^n \lambda_i =1
\end{align*}
where, last part of the equation follows from that fact that vectors $\boldsymbol{x}_i$ are unit vectors. Further, if $(\boldsymbol{\lambda},\boldsymbol{\mu})$ is a feasible solution for (Relaxed Dual) then we must also have
\begin{eqnarray}
\sum_{j=1}^{k} \mu_j \le \sum_{j=1}^{d} \mu_j = Tr(\mathbf{M}) =1 \label{upper_bound_eigenvalues_sum}
\end{eqnarray}
where, $\mu_{(k+1)},\ldots,\mu_d$ are the bottom $(d-k)$ eigenvalues of the matrix $M$.  This proves the part (B) of the lemma.\\\\
{\bf Part (C):} This part follows from the standard results in linear algebra.\\\\
{\bf Part (D):}  Let $(\boldsymbol{\lambda}^*,\boldsymbol{\mu}^*)$ be the point of optimality for (Relaxed Dual) problem. Let $\boldsymbol{v}_1^*,\ldots,\boldsymbol{v}_k^*$ be the top-$k$ eigenvectors of the matrix $\mathbf{M}=\sum_{i=1}^{n}\lambda_i^*\boldsymbol{x}_i\boldsymbol{x}_i^{\top}=\boldsymbol{\lambda}^* \mathbf{X}^{\top}\mathbf{X}$. Because $(\boldsymbol{\lambda}^*,\boldsymbol{\mu}^*)$  is an optimal solution for (Relaxed Dual) problem, it must be a feasible solution also for the same problem. This would mean that $\boldsymbol{\mu}_j^*,\;\forall j\in[k]$ must be top-$j^{th}$ eigenvalue of the matrix $\mathbf{M}=\sum_{i=1}^{n}\lambda_i^*\boldsymbol{x}_i\boldsymbol{x}_i^{\top}$. Therefore, by Lemma \ref{dual_lemma}, we can say that the vectors $\boldsymbol{v}_1^*,\ldots,\boldsymbol{v}_k^*$ must minimize the Lagrangian function $L(\boldsymbol{v}_1,\ldots,\boldsymbol{v}_k,\;Z,\; \lambda_1^*,\ldots,\lambda_n^*,\mu_1^*,\ldots,\mu_k^*)$. The feasibility of the vectors $\boldsymbol{v}_1^*,\ldots,\boldsymbol{v}_k^*$ for (Primal) and (Relaxed Primal) problems is trivial due to the fact these vectors are orthonormal basis vectors. The first part of the inequality $p^* \le \text{\LARGE{$\epsilon$}}^* \le 1$ is trivial because $p^*$ is optimal. The second part of this inequality follows from the second inequality given in part (B) of the this lemma.
\end{proof}
\subsection{Approximate Solution of \ref{Primal} Problem} 
Recall that \ref{Primal} is a non-convex optimization problem. We analyze the approximation factor of the following version of our algorithm. \\
\rule{\columnwidth}{1pt}
{\bf Approximation Algorithm:}  
\begin{enumerate}
\item We first find an optimal solution $(\boldsymbol{\lambda}^*,\boldsymbol{\mu}^*)$ for \ref{Dual_Relaxed_Primal} problem. 
\item Next, we find the top-$k$ eigenvectors $\boldsymbol{v}_1^*,\ldots,\boldsymbol{v}_k^*$ of $M=\sum_{i=1}^{n}\lambda_i^*\boldsymbol{x}_i\boldsymbol{x}_i^{\top}$ and treat them as an approximate solution for the original \ref{Primal} problem. Note, these eigenvectors constitute a feasible solution for the \ref{Primal} as well as \ref{Relaxed_Primal} problems (as shown in Lemma \ref{dual_lemma1}). 
\end{enumerate}
\rule{\columnwidth}{1pt} 
In this section, we try to develop a theoretical bound on the quality of such an approximate solution for the original \ref{Primal}. For this, note that the objective function value for both \ref{Relaxed_Primal} and \ref{Primal} problems is identical for the feasible solution $(\boldsymbol{v}_1^*,\ldots,\boldsymbol{v}_k^*)$ obtained by the method suggested above. Moreover, this value is given by $\text{\LARGE{$\epsilon$}}_{ALG}=\underset{i=1,\ldots,n}{\max} \phi_i$, where $\phi_i=1-\lVert {\mathbf{V}^*} ^{\top} \boldsymbol{x}_i\rVert_2^2$ and $\mathbf{V}^*$ is a matrix comprising top-$k$ eigenvectors $\boldsymbol{v}_1^*,\ldots,\boldsymbol{v}_k^*$ of the matrix $\mathbf{M}$ as its columns. The following inequality follows trivially from above fact and Lemma \ref{dual_lemma1} part (A).
\begin{eqnarray}
\widehat{p^*} \le p^* \le \text{\LARGE{$\epsilon$}}_{ALG} \le 1 \label{approx_ratio1}
\end{eqnarray} 
Further, note that for above algorithm, the optimal objective function value for \ref{Dual_Relaxed_Primal} would be $\widehat{d^*}=1-\sum_{j=1}^{k}\mu_j^*$, where $\mu_j^*$ is the $j^{th}$-top eigenvalue of the matrix $\mathbf{M}$. By combining the Inequality (\ref{approx_ratio1}) with Lemma \ref{dual_lemma1} part (A), we can say that
\begin{eqnarray}
\frac{\text{\LARGE{$\epsilon$}}_{ALG}}{p^*} \le \frac{\text{\LARGE{$\epsilon$}}_{ALG}}{\widehat{d^*}} \le \frac{1}{1-\sum_{j=1}^{k}\mu_j^*} \label{approx_ratio2}
\end{eqnarray} 
In order to obtain a meaningful upper bound on the above inequality, we recall the definition of matrix $\mathbf{X}$ whose size is $n \times d$, and whose rows are unit length data vectors $\boldsymbol{x}_1,\boldsymbol{x}_2,\dots, \boldsymbol{x}_n$. Suppose $\sigma_1 \ge \sigma_2 \ge \ldots \ge \sigma_d \ge 0$ are singular values of the matrix $\mathbf{X}$ out of which only $\ell$ are non-zero, where $\ell \le d$ is the rank of the data matrix $\mathbf{X}$. The following theorem gives us the bound on the approximation ratio of our proposed algorithm, where $\kappa$ is the ratio of the highest singular value to the lowest non-zero singular value for the matrix $\mathbf{X}$, i.e. $\kappa = \sigma_1/\sigma_{\ell}$.
\begin{theorem}\label{approx_theorem}
The approximation algorithm described above offers the following approximation guarantees.
\begin{eqnarray}
\frac{\text{\LARGE{$\epsilon$}}_{ALG}}{p^*} \le \frac{1}{1- \frac{\sigma_1^2}{n}}=  {\left(1-\frac{\sigma_1^2}{\sigma_1^2+\sigma_2^2+\ldots+\sigma_l^2}\right)}^{-1} \le \frac{1}{1- \frac{\kappa^2}{l}} 
\end{eqnarray}
\end{theorem}
\begin{proof}
As per Inequality (\ref{approx_ratio2}), in order to prove the claim of this theorem, it suffices to get an upper bound on the quantity $\sum_{j=1}^{k}\mu_j^*$. Thus, our goal is to get an upper bound on the sum of top-$k$ singular values of the matrix $\mathbf{M}$. For this, we note that $\mathbf{M} = \mathbf{X^{\top}}\mathbf{\Lambda}\mathbf{X}$, where $\mathbf{\Lambda}=diag(\lambda_1,\ldots,\lambda_n)$, $\sum_{i=1}^{n}\lambda_i=1$, and $\lambda_i \ge 0\ \forall i$. 
Let the SVD of the matrix $\mathbf{X^{\top}}$ be given as follows.
\begin{eqnarray}
\mathbf{X^{\top}} = \mathbf{U}\mathbf{\Sigma}\mathbf{\widetilde{V}^{\top}}
\end{eqnarray}
where $\mathbf{U}$ is a $d \times d$ orthogonal matrix, $\mathbf{\Sigma}$ is a $d \times d$ diagonal matrix containing $\sigma_1, \sigma_2, \ldots, \sigma_d$ on its diagonal, and  $\mathbf{\widetilde{V}}$ is an $n \times d$ matrix having orthonormal column vectors $\boldsymbol{\widetilde{v}_1},\boldsymbol{\widetilde{v}_2},\ldots,\boldsymbol{\widetilde{v}_d}$. Recall that, we have used symbol $\mathbf{V}$ to denote the solution of the primal problem which is a $d \times k$ matrix and hence we are using a different symbol $\mathbf{\widetilde{V}^{\top}}$ to denote the left singular vectors of the matrix $\mathbf{X^{\top}}$.

By using SVD of the matrix $\mathbf{X^{\top}}$, we can rewrite the expression for matrix $\mathbf{M}$ as below.
\begin{eqnarray}
\mathbf{M} = \mathbf{U}\mathbf{\Sigma}\mathbf{\widetilde{V}^{\top}} \mathbf{\Lambda}\mathbf{\widetilde{V}}\mathbf{\Sigma} \mathbf{U^{\top}} \label{factor_M}
\end{eqnarray} 
Recall that matrix $\mathbf{M}$ is a PSD matrix and hence, we must have
\begin{eqnarray}
\sum_{j=1}^{k}\mu_j \le \sum_{j=1}^{d} \mu_j = \text{Tr}(\mathbf{M})
\end{eqnarray}
By making use of Equation (\ref{factor_M}), above inequality can be written as
\begin{eqnarray}
\sum_{j=1}^{k}\mu_j & \le & \text{Tr}(\mathbf{U}\mathbf{\Sigma}\mathbf{\widetilde{V}^{\top}} \mathbf{\Lambda}\mathbf{\widetilde{V}}\mathbf{\Sigma} \mathbf{U^{\top}}) =  \text{Tr}(\mathbf{\Sigma}\mathbf{\widetilde{V}^{\top}} \mathbf{\Lambda}\mathbf{\widetilde{V}}\mathbf{\Sigma})
\end{eqnarray} 
From the definition of \ref{Dual_Relaxed_Primal}, we can write
\begin{eqnarray}
\sum_{j=1}^{k}\mu_j^* = \underset{\Lambda}{\text{Min}}\;\left(\sum_{j=1}^{k}\mu_j\right) \le \underset{\Lambda}{\text{Min}}\; \left(\text{Tr}(\mathbf{\Sigma}\mathbf{\widetilde{V}^{\top}} \mathbf{\Lambda}\mathbf{\widetilde{V}}\mathbf{\Sigma})\right) \label{upperbound_on_sum_mu_k}
\end{eqnarray}
In lieu of the fact that $\mathbf{\Lambda}$ is a diagonal matrix having $\lambda_1,\ldots,\lambda_n$ on its diagonal, where $\sum_{i} \lambda_i=1$ and $\lambda \ge 0$, and  $\boldsymbol{\widetilde{v}_1},\boldsymbol{\widetilde{v}_2},\ldots,\boldsymbol{\widetilde{v}_d}$ are orthonormal column vectors of the matrix $\mathbf{\widetilde{V}}$, it is not difficult to verify that $i^{th}$ diagonal entry of the matrix $\mathbf{\Sigma}\mathbf{\widetilde{V}^{\top}} \mathbf{\Lambda}\mathbf{\widetilde{V}}\mathbf{\Sigma}$ can be given by 
\begin{eqnarray} 
\left[\mathbf{\Sigma}\mathbf{\widetilde{V}^{\top}} \mathbf{\Lambda}\mathbf{\widetilde{V}}\mathbf{\Sigma}\right]_{ii} = \sigma_i^2\left(\widetilde{v}^2_{i1} \lambda_1 + \widetilde{v}^2_{i2} \lambda_2 +\ldots+\widetilde{v}^2_{in} \lambda_n\right)
\end{eqnarray}
where, we follow the convention that
\begin{eqnarray}
\mathbf{\widetilde{V}^{\top}} = 
\left[ 
	{\begin{array}{c}
   \boldsymbol{\widetilde{v}^{\top}_1}\\\\
   \boldsymbol{\widetilde{v}^{\top}_2}\\\\
   ... \\\\
   \boldsymbol{\widetilde{v}^{\top}_d}
  \end{array} } 
  \right] = \left[ 
	{\begin{array}{cccc}
   {\widetilde{v}_{11}}&{\widetilde{v}_{12}}&\ldots&{\widetilde{v}_{1n}}\\\\
   {\widetilde{v}_{21}}&{\widetilde{v}_{22}}&\ldots&{\widetilde{v}_{2n}}\\\\
   ... \\\\
   {\widetilde{v}_{d1}}&{\widetilde{v}_{d2}}&\ldots&{\widetilde{v}_{dn}}
  \end{array} } 
  \right]
\end{eqnarray}
This would imply that
\begin{eqnarray} 
&&\text{Tr}\left(\mathbf{\Sigma}\mathbf{\widetilde{V}^{\top}} \mathbf{\Lambda}\mathbf{\widetilde{V}}\mathbf{\Sigma}\right) = \sum_{i=1}^{d}\sigma_i^2\left(\widetilde{v}^2_{i1} \lambda_1 + \widetilde{v}^2_{i2} \lambda_2 +\ldots+\widetilde{v}^2_{in} \lambda_n\right)\\
&=& \lambda_1 \left(\sum_{i=1}^{d}\sigma_i^2 \widetilde{v}^2_{i1}\right)+ \lambda_2\left(\sum_{i=1}^{d}\sigma_i^2 \widetilde{v}^2_{i2}\right) + \ldots + \lambda_n\left(\sum_{i=1}^{d} \sigma_i^2 \widetilde{v}^2_{in}\right) \label{Trace_equality}
\end{eqnarray}
Substituting Equation (\ref{Trace_equality}) back into the Inequality (\ref{upperbound_on_sum_mu_k}) would give us the following inequality.
\begin{eqnarray*}
\sum_{i=1}^{k}\mu_j^* &\le& \underset{j=1,\ldots,n}{\text{Min}} \; \sum_{i=1}^{d} \sigma_i^2 \widetilde{v}^2_{ij} \le \underset{j=1,\ldots,n}{\text{Min}} \; \sum_{i=1}^{d} \sigma_i^2 \widetilde{v}^2_{ij} \\
 &=& \sigma_1^2\;\; \underset{j=1,\ldots,n}{\text{Min}} \; {\left(\widetilde{v}^2_{1j}+\widetilde{v}^2_{2j}+\ldots+\widetilde{v}^2_{dj}\right)}
\end{eqnarray*}
Again, because of the fact that $\boldsymbol{\widetilde{v}_1},\boldsymbol{\widetilde{v}_2},\ldots,\boldsymbol{\widetilde{v}_d}$ are orthonormal column vectors of the matrix $\mathbf{\widetilde{V}}$, we can write
\begin{eqnarray}
\sum_{j=1}^{n}{\left(\widetilde{v}^2_{1j}+\widetilde{v}^2_{2j}+\ldots+\widetilde{v}^2_{dj}\right)} =1
\end{eqnarray}
This implies that
\begin{eqnarray}
\sum_{i=1}^{k}\mu_j^* &\le& \sigma_1^2\;\; \underset{j=1,\ldots,n}{\text{Min}} \; {\left(\widetilde{v}^2_{1j}+\widetilde{v}^2_{2j}+\ldots+\widetilde{v}^2_{dj}\right)} \le \frac{\sigma_1^2}{n} \label{final_inequality_1}
\end{eqnarray}
This gives us the first part of the inequality in the theorem. In order to get the second part of the inequality, we note that the following relation is easy to verify.
\begin{eqnarray}
n = \text{Tr}\left(\mathbf{X} \mathbf{X}^{\top} \right) = \text{Tr}\left(\mathbf{V} \mathbf{\Sigma}\mathbf{U}^{\top}\mathbf{U}\mathbf{\Sigma}\mathbf{V}^{\top} \right) = \sigma_1^2+\sigma_2^2+\ldots+\sigma_l^2 \label{final_inequality_2}
\end{eqnarray}
Substituting above equation into the Inequality (\ref{final_inequality_1}) would yield the second desired inequality in the theorem's statement. The last inequality in the theorem statement follows from the definition of $\kappa$ for the matrix $\mathbf{X}^{\top}$ as given by
\begin{eqnarray}
\kappa = \frac{\sigma_1}{\sigma_l}
\end{eqnarray} 
\end{proof}
Below are some interesting and useful insights about our algorithm that once can easily derive with the help of Theorem \ref{approx_theorem}.  
\begin{itemize}
\item The above theorem bounds the gap between primal objective value corresponding to our approximate solution and the unknown optimal primal objective value.
\item The approximation bound given in Theorem \ref{approx_theorem} depends on the spectral properties of the input dataset. For example, if the given dataset $\mathbf{X}$ is well behaved, that is $\kappa=1$, then we obtain the approximation factor of $\frac{l}{l-1}$ which is bounded above by $2$ assuming $\ell \ge 2$.  
\item Further, in such a case, this bound starts approaching towards $1$ as the rank of the data matrix increases (of course, the dimension $d$ of the data has to increases first for $\ell$ to increase). Thus, we can say that for well behaved dataset of full rank in very large dimensional spaces, our algorithm offers nearly optimal solution for the problem of length preserving orthogonal projections of the data.   
\end{itemize}

\section{Projected Gradient Ascent}
We now give a provable algorithm to get an optimal solution of the \ref{Dual_Relaxed_Primal} problem. Since the dual formulation is always a convex program~\cite{Boyd04}, the objective function of \ref{Dual_Relaxed_Primal} is concave in the dual variables and the feasible set formed by the constraint set is a convex set $\textbf{C} \subset \mathbb{R}^n$. 
In each iteration, our algorithm essentially performs a gradient ascent update on dual variables $\lambda_i, i=1, \ldots,n$ and then projects the update back onto feasible convex set $\textbf{C}$. Note that we need not worry about dual variables $\mu_j$ because for any assignment of the variables $\lambda_i$, the values of the variables $\mu_j$ gets determined automatically because of first constraint in \ref{Dual_Relaxed_Primal} problem. Therefore, we perform gradient ascent step only on $\lambda_i$ variables.

Observe that feasible region of $\lambda_i$ variables forms the standard probability simplex in $\mathbb{R}^n$. Because of this, we make use of the {\em Simplex Projection algorithm} proposed by \cite{Wang13} for the purpose of performing the projection step in each iteration of our algorithm. This algorithm 
runs in time $\mathcal{O}(d\log(d))$ time. We call the projection step of our algorithm as (\textbf{Proj}\textsubscript{\textbf{C}}). 

The pseudo code for the projected gradient ascent algorithm is given in the form of Algorithm \ref{Proj_Grad_Algo}. In the line number 5 (and also 7) of this code, we compute the \ref{Primal} objective function value $\text{\LARGE{$\epsilon$}}$ for two different dual feasible solutions, namely ${\boldsymbol\lambda}^{(t+1)}$ and ${\boldsymbol\lambda}^{best}$. Note, the last \texttt{if-else} statement identifies better of the two solutions (in terms of the primal objective function value). 
\begin{algorithm}[!htb]
\SetKwInOut{Input}{Input}
\SetKwInOut{Output}{Output}
\SetKwInOut{Initialization}{Initialize}
\Input{$x_i, i =1, \ldots, n$}
\BlankLine
\Initialization{$\lambda_i^0 =\lambda_i^{best}= 1/n$ $\forall i \in [n]$}
\BlankLine
\caption{Projected Gradient Ascent\label{Proj_Grad_Algo}}
\For{$t = 1, \ldots, T$}
{
	Compute $\nabla g(\boldsymbol{\lambda},\boldsymbol{\mu}))$ for $(\boldsymbol{\lambda},\boldsymbol{\mu} ) = \left({\boldsymbol{\lambda}}^{(t)},{\boldsymbol{\mu}}^{(t)}\right)$ by making use of Lemma \ref{lemma_dual_gradient}\;
    \BlankLine
	$\widetilde{\boldsymbol{\lambda}}^{(t+1)} \leftarrow {\boldsymbol{\lambda}}^{(t)} + \eta \nabla g({\boldsymbol{\lambda}}^{(t)},{\boldsymbol{\mu}}^{(t)})$\;
\BlankLine
${\boldsymbol\lambda}^{(t+1)} \leftarrow \textbf{Proj}_{\textbf{C}}\left(\widetilde{\boldsymbol\lambda}^{(t+1)}\right)$\;
\BlankLine
\If
	{$\left(\text{\LARGE{$\epsilon$}} \text{ for } {\boldsymbol\lambda}^{(t+1)}\right)$ < $\left(\text{\LARGE{$\epsilon$}} \text{ for } {\boldsymbol\lambda}^{best}\right)$}
	{
		\BlankLine
        ${\boldsymbol\lambda}^{best} = {\boldsymbol\lambda}^{(t+1)}$
	}
}	
\eIf{$\left(\text{\LARGE{$\epsilon$}} \text{ for } {\boldsymbol\lambda}^{best}\right) < \left(\text{\LARGE{$\epsilon$}} \text{ for } {\frac{1}{T}\sum_{t=1}^{T}\boldsymbol\lambda}^{(t)}\right)$}
{
	\BlankLine
    \Output{${\boldsymbol\lambda}^{best}$}
}
{
	\Output{${\frac{1}{T}\sum_{t=1}^{T}\boldsymbol\lambda}^{(t)}$}
}
\end{algorithm}
The routine $\textbf{Proj}_{\textbf{C}}(\cdot)$ is given in Algorithm \ref{Proj_C}.
\IncMargin{1em}
\begin{algorithm}[!htb]
\SetKwInOut{Input}{Input}
\SetKwInOut{Output}{Output}
\SetKwInOut{Initialization}{Initialize}
\SetKwInOut{Objective}{Objective}
\Input{$\boldsymbol\lambda \in \mathbb{R}^n$}
\BlankLine
\Objective{Find $\underset{\widetilde{\boldsymbol\lambda} \in \textbf{C}}{argmin}\| \widetilde{\boldsymbol\lambda} - \boldsymbol\lambda \|_{2}$}
\BlankLine
\caption{$\textbf{Proj}_{\textbf{C}}(\cdot)$\label{Proj_C}}
Sort coordinated of $\boldsymbol\lambda$ into $\lambda_{(1)} \geq \ldots \lambda_{(n)} $\;
\BlankLine
$\rho \leftarrow \underset{j\in[n]}{\max}\left\{\lambda_{(j)} + \left(1 - \sum_{i=1}^j \lambda_{(i)}\right)/j\right\}$\;
\BlankLine
$\alpha \leftarrow (1/\rho)\left(1-\sum_{i=1}^{\rho} \lambda_{(i)}\right)$\;
\BlankLine
\Output{$\widetilde{\boldsymbol\lambda}$ such that $\widetilde{\lambda_i} = \max(\lambda_i + \alpha,0)\;\forall i \in [n]$}
\end{algorithm}

Now, we present a key lemma related to the gradient of (Relaxed Dual) objective.
\begin{lemma}[{\bf Dual Gradient}] \label{lemma_dual_gradient} If $\boldsymbol{\lambda}^{(t)}$ is not an optimal solution of \ref{Dual_Relaxed_Primal} problem then for $\ell =1,\ldots,n$, the $\ell^{th}$ coordinate of the gradient vector $\nabla g\left(\boldsymbol{\lambda}^{(t)}, \boldsymbol{\mu}^{(t)}\right)$ is given by
\begin{eqnarray}
{\partial g\left(\boldsymbol{\lambda}^{(t)}, \boldsymbol{\mu}^{(t)}\right)}/{\partial {\lambda}_{\ell}^{(t)}}
&=& -\sum\nolimits_{j=1}^{k} \lVert \boldsymbol{x}_{\ell}^{\top} \boldsymbol{v}_j^{(t)}\rVert^2
\end{eqnarray}
\end{lemma}
\begin{proof} Note, the gradient vector $\nabla g\left(\boldsymbol{\lambda}^{(t)}, \boldsymbol{\mu}^{(t)}\right)$ would be  
\begin{eqnarray*}
{\partial g\left(\boldsymbol{\lambda}^{(t)}, \boldsymbol{\mu}^{(t)}\right)}/{\partial \boldsymbol{\lambda}^{(t)}}= {\partial \left(1-\sum\nolimits_{j=1}^{k} \mu_j^{(t)}\right)}/{\partial \boldsymbol{\lambda}^{(t)}}
\end{eqnarray*}
The ${\ell}^{th}$ coordinate of this ascent direction can be given by
\begin{eqnarray*}
{\partial g\left(\boldsymbol{\lambda}^{(t)}, \boldsymbol{\mu}^{(t)}\right)}/{\partial {\lambda}_{\ell}^{(t)}} = - {\partial \left(\sum\nolimits_{j=1}^{k} \mu_j^{(t)}\right)}/{\partial {\lambda}_{\ell}^{(t)}},\forall {\ell}\in [n]
\end{eqnarray*}
Let $\boldsymbol{v}_1^{(t)}, \ldots \boldsymbol{v}_k^{(t)}$ are the top-$k$ eigenvectors of the matrix $\mathbf{M}(\boldsymbol{\lambda}^{(t)})$ then above equation can be written as
\begin{eqnarray*}
{\partial g\left(\boldsymbol{\lambda}^{(t)}, \boldsymbol{\mu}^{(t)}\right)}/{\partial {\lambda}_{\ell}^{(t)}} = - {\partial \left(\sum\nolimits_{j=1}^{k} {\boldsymbol{v}_j^{(t)}}^{\top}\mathbf{M}(\boldsymbol{\lambda}^{(t)})\boldsymbol{v}_j^{(t)}\right)}/{\partial {\lambda}_{\ell}^{(t)}}
\end{eqnarray*}
Recall that $\mathbf{M}(\boldsymbol{\lambda}^{(t)}) = \sum\nolimits_{i=1}^{n}\lambda_i^{(t)}\boldsymbol{x}_i\boldsymbol{x}_i^{\top}$. Substituting this expression for $\mathbf{M}(\boldsymbol{\lambda}^{(t)})$ in the previous equation gives us the following relation.
\begin{eqnarray*}
&{\partial g\left(\boldsymbol{\lambda}^{(t)}, \boldsymbol{\mu}^{(t)}\right)}/{\partial {\lambda}_{\ell}^{(t)}} = -\sum\nolimits _{j=1}^{k} {\boldsymbol{v}_j^{(t)}}^{\top} (\boldsymbol{x}_{\ell}\boldsymbol{x}_{\ell}^{\top})  {\boldsymbol{v}_j^{(t)}}\\
& - 2\sum_{j=1}^{k} \sum_{i=1}^{n}\lambda_i^{(t)}  \left({\partial \boldsymbol{v}_j^{(t)}}/{\partial \lambda_{\ell}^{(t)}}\right)^{\top}(\boldsymbol{x}_{i}\boldsymbol{x}_{i}^{\top})\boldsymbol{v}_j^{(t)}\\
&= -\sum\nolimits_{j=1}^{k} \lVert \boldsymbol{x}_{\ell}^{\top} \boldsymbol{v}_j^{(t)}\rVert^2 - 2\sum\nolimits_{j=1}^{k} \left({\partial \boldsymbol{v}_j^{(t)}}/{\partial \lambda_{\ell}^{(t)}}\right)^{\top}\mathbf{M}(\boldsymbol{\lambda}^{(t)})\boldsymbol{v}_j^{(t)}\\
&= -\sum\nolimits_{j=1}^{k} \lVert \boldsymbol{x}_{\ell}^{\top} \boldsymbol{v}_j^{(t)}\rVert^2 - 2\sum\nolimits_{j=1}^{k} \mu_j^{(t)}\left({\partial \boldsymbol{v}_j^{(t)}}/{\partial \lambda_{\ell}^{(t)}}\right)^{\top}\boldsymbol{v}_j^{(t)}
\end{eqnarray*}
The last term in the above expression would be zero because
\begin{eqnarray*}
\left({\partial \boldsymbol{v}_j^{(t)}}/{\partial \lambda_{\ell}^{(t)}}\right)^{\top}\boldsymbol{v}_j^{(t)}=0,\;\forall j\in [k];\;\forall \ell\in [n]
\end{eqnarray*}
This facts follows from the another fact that $\boldsymbol{v}_j^{(t)},\;\forall j=1,\ldots,k$ are eigenvectors and hence we must have ${\boldsymbol{v}_j^{(t)}}^{\top}\boldsymbol{v}_j^{(t)}=1$. Now differentiating this relation on both the sides with respect to $\lambda_{\ell}$ would give us the desired fact.
\end{proof}

\subsection{Convergence Guarantees}
In this section, we show that Projected Gradient Ascent algorithm given earlier converges to the optimal solution. For this, we just recall here a known result (Theorem \ref{bubeck_thm}) in the literature of convex optimization. Readers can refer to Theorem 3.2. in \cite{Bubeck15} for the proof of this result. Here, we have adopted this result for the case of concave functions.
\begin{theorem}[\cite{Bubeck15}] \label{bubeck_thm}
Consider a convex optimization problem, where objective function $f(\boldsymbol{x})$ is a concave function that needs to be maximized over a feasible convex set $\boldsymbol{C}\in \mathbb{R}^n$. Further, let $f(\boldsymbol{x})$ and $\boldsymbol{C}$ satisfy the following conditions, where $L, D, \eta >0$ are constants. 
\begin{enumerate}
\item $\lVert \nabla f(\boldsymbol x) \rVert_{2} \leq $ L
\item $ \lVert\boldsymbol{x}-\boldsymbol{y}\rVert_{2} \leq D; \forall \boldsymbol{x}, \boldsymbol{y} \in C$
\item $\eta = D/(L\sqrt{T})$
\end{enumerate}
Let ${\boldsymbol x}^*$ be the optimal solution of this problem. If we run the projected gradient ascent algorithm on this problem for $T$ iterations with step size $\eta$ then following bound holds. 
\begin{eqnarray}
f(\boldsymbol{x}^*) - \frac{LD}{\sqrt{T}} &\le & f\left(\frac{\sum_{t=1}^T \boldsymbol{x}_{t}}{T}\right)
\end{eqnarray} 
\end{theorem}

The above theorem essentially states that the average of all iterates can get arbitrarily close to the maximum as we increase $T$. 
One can verify that our \ref{Dual_Relaxed_Primal} problem satisfies all the conditions of the above theorem with the following values of the parameters $D$ and $L$.
\begin{enumerate}
\item In the feasible set of the \ref{Dual_Relaxed_Primal} problem, the variables $\boldsymbol\mu$ get uniquely frozen once we freeze the values of $\boldsymbol\lambda$. Therefore, we can view the dual function $g(\boldsymbol\lambda, \boldsymbol\mu)$ as function of $\boldsymbol\lambda$ only, that is $g(\boldsymbol\lambda, \boldsymbol\mu)=g(\boldsymbol\mu(\boldsymbol\lambda))$. 
\item  According to Lemma \ref{lemma_dual_gradient}, for the $\ell^{th}$-coordinate of the gradient of dual function, we have $|\nabla g(\boldsymbol\mu(\boldsymbol\lambda))_{\ell} | = \sum_{j=1}^k \lVert \boldsymbol{x}_{\ell}^{\top} \boldsymbol{v}_j \rVert_2^2 \leq \sum_{j=1}^d \lVert \boldsymbol{x}_{\ell}^{\top} \boldsymbol{v}_j \rVert_2^2 = 1$ because all $\left\{\boldsymbol{v}_j\right\}_{j=1}^{d}$ form an orthogonal basis of $\mathbb{R}^d$ and the projection of any unit length data point $\boldsymbol{x}_{\ell}$ on these vectors will have unit length. This implies that $\lVert \nabla g(\boldsymbol\mu(\boldsymbol\lambda)) \rVert_2 \leq \sqrt{n}$ and thus $L = \sqrt{n}$ for our case.
\item The maximum value of $\| x- y\|_2$ for any $x,y \in C$ will be $\sqrt{2}$ as the farthest points on the probability simplex will be any of its two corners and the distance between them will be $\sqrt{2}$. Thus, $D = \sqrt{2}$ for our case.
\end{enumerate}

Note that above convergence guarantee is for the average iterate. In Algorithm 1, we compare the primal objective function value at the average iterate with the value at the best iterate so far (Step 7 of Algorithm 1), and output the one for which the it is lower. Therefore, theoretical guarantees of the Theorem \ref{bubeck_thm} remains valid for the output of Algorithm 1.  
\section{Key Insights of our Algorithm}
In this section, we highlight some important insights regarding the problem and our approach.
\begin{enumerate}
\item The proposed \ref{Dual_Relaxed_Primal} problem has far less number of constraints compared to both the \ref{Primal} and \ref{Relaxed_Primal} problems. Contrary to this, the baseline formulation of \cite{Indyk13} is a SDP relaxation of the \ref{Primal} problem and because of which they have one constraint per data point. As mentioned in our experiments section, this fact is one of the main reasons that the SDP based algorithms proposed by \cite{Indyk13} do not scale as well as our algorithm with increasing number of data points.
\item The main computation in each iteration of our algorithm involve dual gradient computation and projecting dual solution onto a probability simplex. The gradient calculation requires top-$k$ eigenvectors of $M(\lambda)$ which can be computed quickly because $k \ll d$ in any dimensionality reduction problem. Further, we also use a fast algorithm for projection onto the probability simplex as described earlier. All these together make our scheme very fast.
\end{enumerate}
\section{Experiments}
In this section, we present the results of our experiments wherein, we have compared the performance of our algorithm with four baseline algorithms. The first two baselines 
are {\em PCA} and {\em Random Projections} ({\em Random} for short). The other two baselines are both based on the algorithms
given in~\cite{Indyk13}. These are based on semidefinite programming relaxations of the \ref{Primal} problem and are named {\em SDP+RR} and {\em SDP+DR} depending on whether the rounding is randomized (RR) or deterministic (DR). 

All our experiments were performed using MATLAB on a machine having 8-Core Intel i7 processor and 64GB RAM. The details of the datasets used in our experiments and values of various hyper-parameters used by our algorithm are being summarized in Table~\ref{tab:data}. 
\begin{center}
\begin{table*}[!ht]
\centering
\begin{tabular}[!ht]{ccccc} 
\hline \hline
Dataset & Dimension ($d$) & Sample Size ($n$) & Iterations ($T$) & Step Size ($\eta$) \\
\hline\hline
MNIST (Digits $2, 4, 5, 7$) & $784$ & $1035$ & $120$ & $0.004$\\ 
\hline
MNIST (Digits $2,4,5,7$) & $784$ & $5050$ & $120$ & $0.0018$\\
\hline
MNIST (Digits $2,4,5,7$) & $784$ & $10011$ & $120$ & $0.00129$\\
\hline
MNIST (Digits $2,4,5,7$) & $784$ & $50086$ & $120$ & $0.00057$\\
\hline
MNIST (Digits $2,4,5,7$) & $784$ & $100128$ & $120$ & $0.0004$\\
\hline
20 Newsgroup&&&&\\ (Atheism and MS-Windows Misc.) & $8000$ & $1034$ & $120$ & $0.004$ \\
\hline
\hline
\end{tabular}
\caption{Details of the datasets and the hyper-parameters used in our experiments}	
\label{tab:data}
\end{table*}
\end{center}
The goal of our experiments is to compare the performance of our algorithm vis-\`a-vis the performance of the baseline algorithms in terms of (i) quality of the solution (as measured by the value of the \ref{Primal} objective function), and (ii) time taken by the respective algorithms. At this point, we would like to highlight that the time taken by any iterative algorithm depends on its stopping criterion. Given that our algorithm is also iterative in nature, we use the following stopping criterion - {\em We stop our algorithm whenever an appropriate pre decided number of iteration count is completed (specifically, $120$ in our case).} As far as learning rate $(\eta)$ is concerned for our algorithm, after fixing the iteration count, we choose $\eta$ for the given dataset size as suggested by the third bullet point in the statement of Theorem 2.

In what follows, we describe our experimental setup details, results,  and the corresponding insights for each of the dataset separately.
\begin{paragraph}
{\bf MNIST Dataset:} MNIST\cite{lecun-mnisthandwrittendigit-2010} dataset comprises a collection of images for handwritten digits in the range of $0-9$. Each image in the MNIST dataset is a $28 \times 28$ gray scale image matrix. For our experimental purpose,  we flattened each of these images into a $784$ dimensional vector.  
\begin{figure*} 
\begin{multicols}{3}
    \includegraphics[scale=0.22]{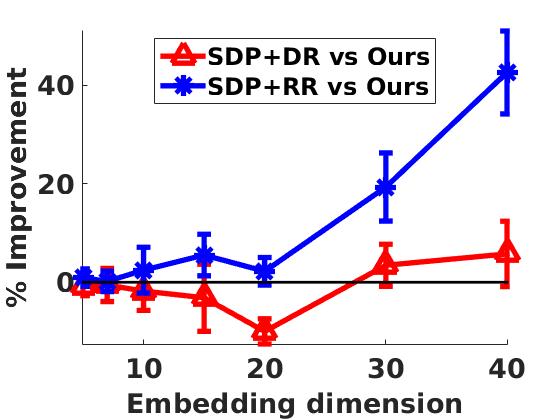}\par 
        \includegraphics[scale=0.22]{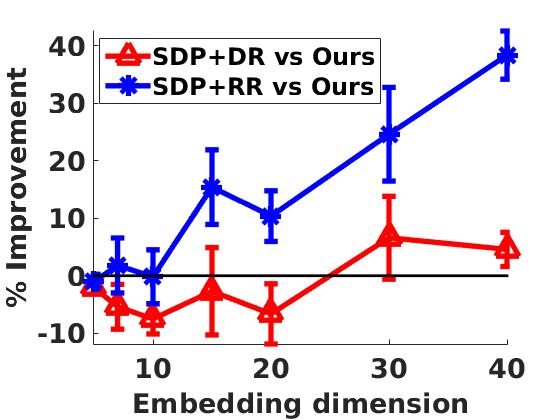}\par 
        \includegraphics[scale=0.22]{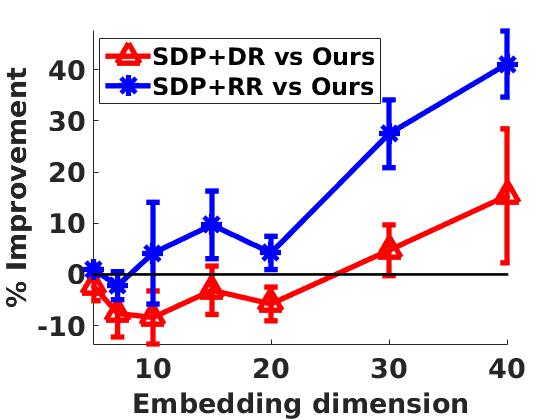}\par 
        \includegraphics[scale=0.22]{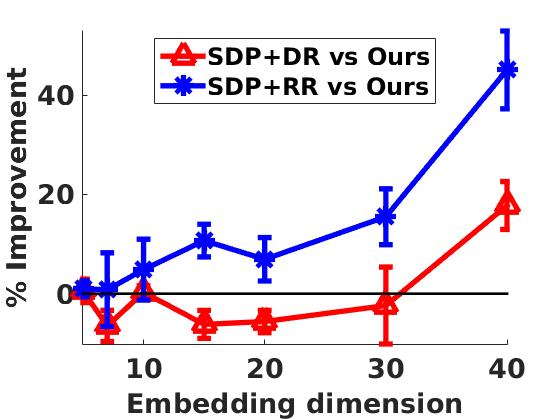}\par 
    \includegraphics[scale=0.22]{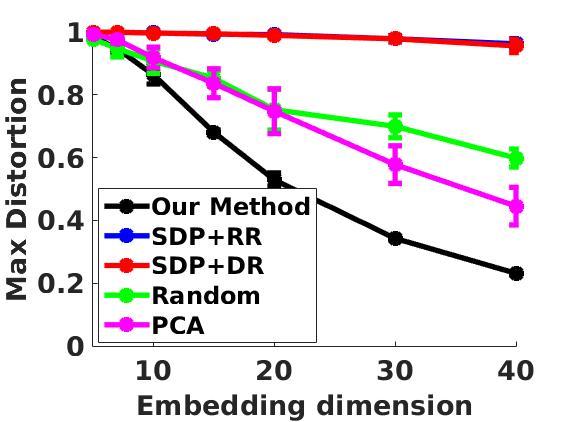}
        \includegraphics[scale=0.22]{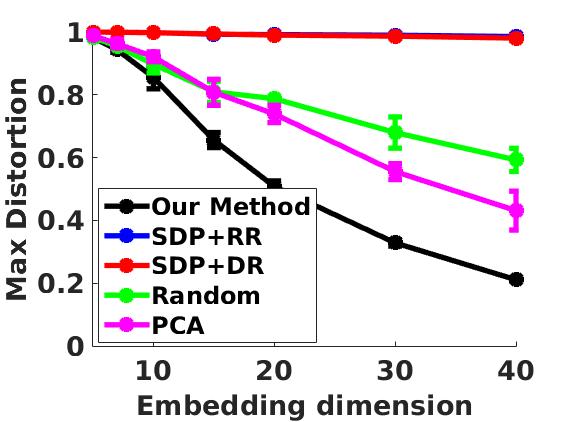}
\includegraphics[scale=0.22]{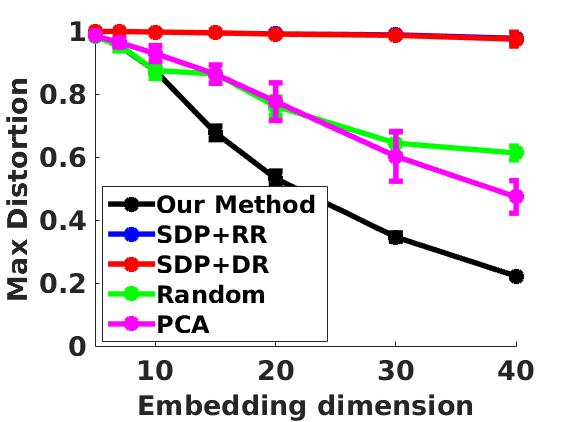}
        \includegraphics[scale=0.22]{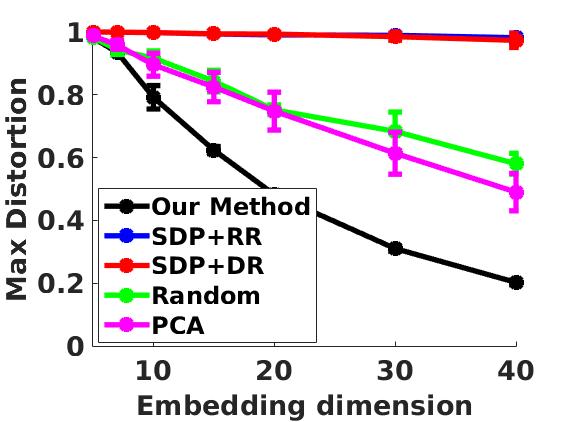}
        \includegraphics[scale=0.22]{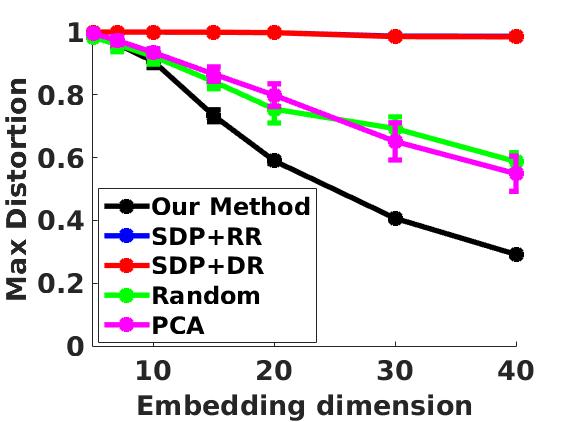}
        \includegraphics[scale=0.22]{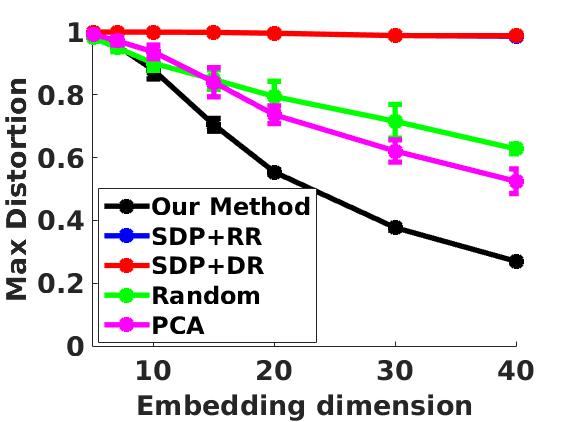}
        \includegraphics[scale=0.22]{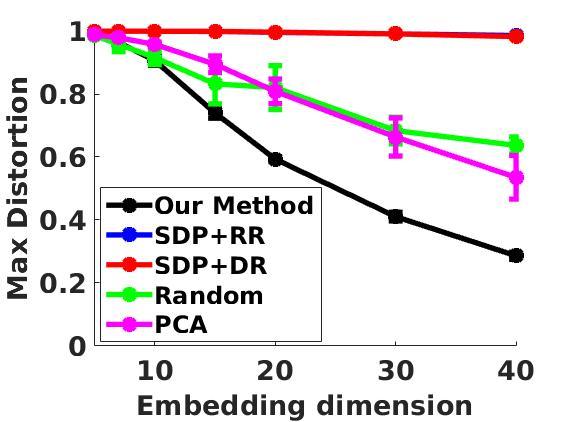}
 \includegraphics[scale=0.22]{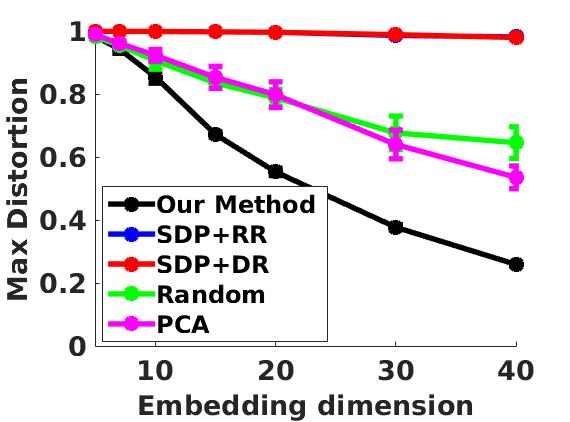}
\end{multicols}
\caption{\label{MNIST_Plots1} Performance of the proposed algorithm relative to baseline methods on MNIST dataset. Columns 1, 2, and 3 correspond to dataset size of $1K$ $5K$, and $10K$, respectively. The rows from top to bottom correspond to MNIST digit $2, 4, 5$, and $7$, respectively.}
\end{figure*}
\begin{figure*} 
\begin{multicols}{3}
    \includegraphics[scale=0.22]{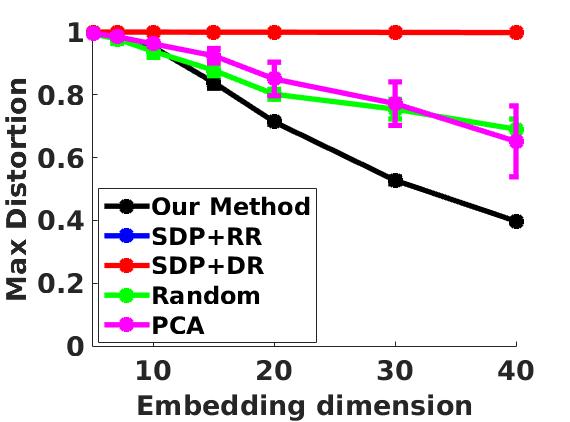}\par 
        \includegraphics[scale=0.22]{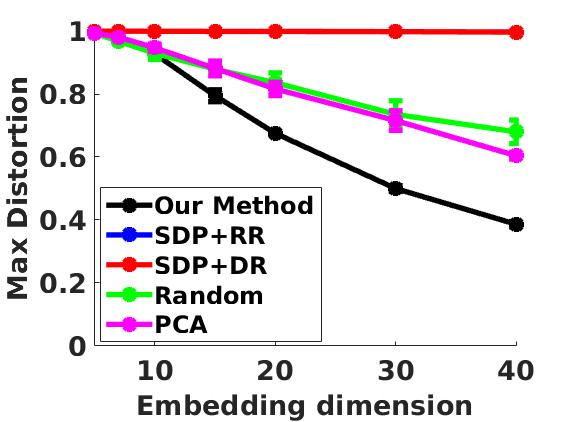}\par 
        \includegraphics[scale=0.22]{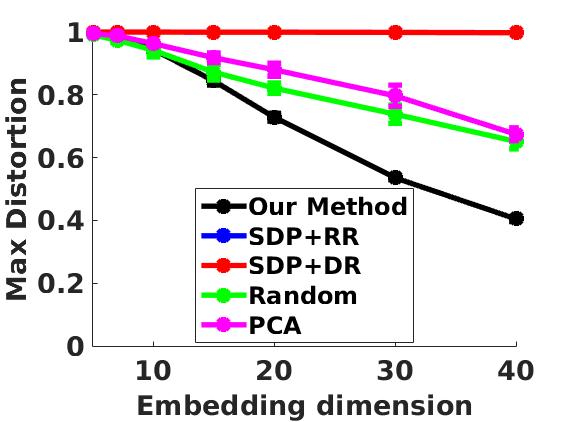}\par 
        \includegraphics[scale=0.22]{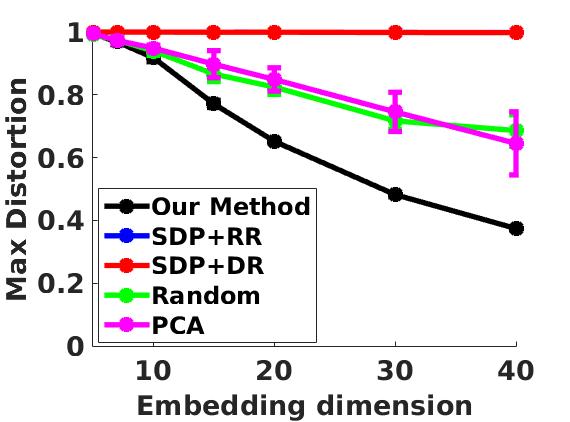}\par 
    \includegraphics[scale=0.22]{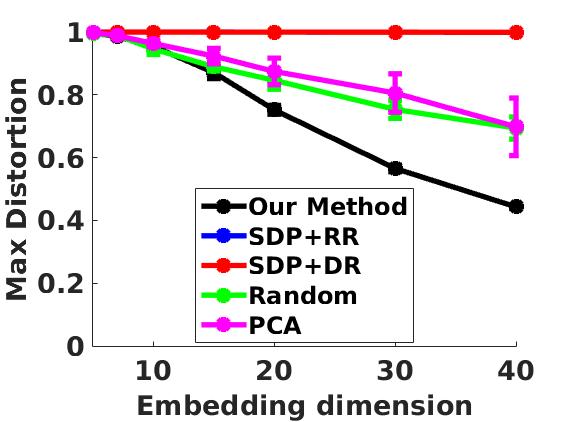}
        \includegraphics[scale=0.22]{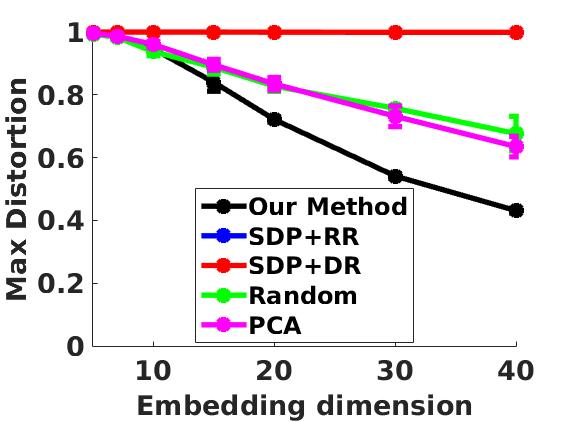}
\includegraphics[scale=0.22]{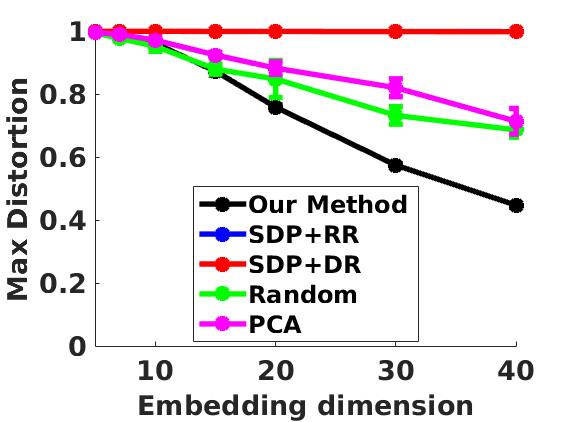}
\vspace{4cm}             \includegraphics[scale=0.22]{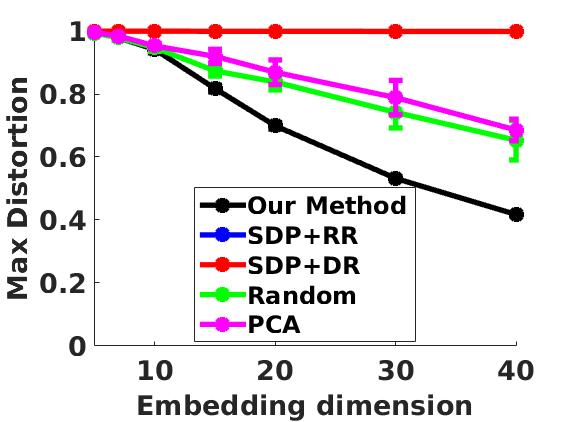}
 \includegraphics[scale=0.22]{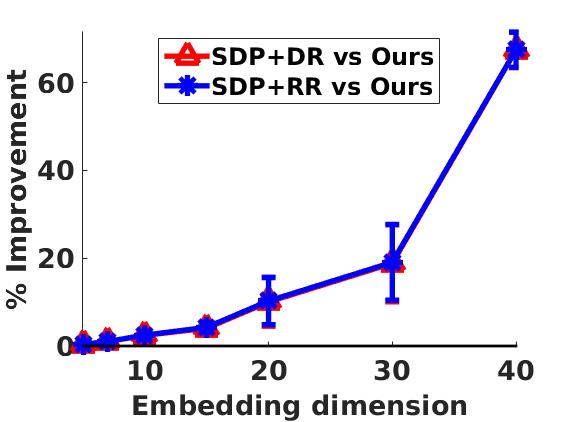}
\includegraphics[scale=0.22]{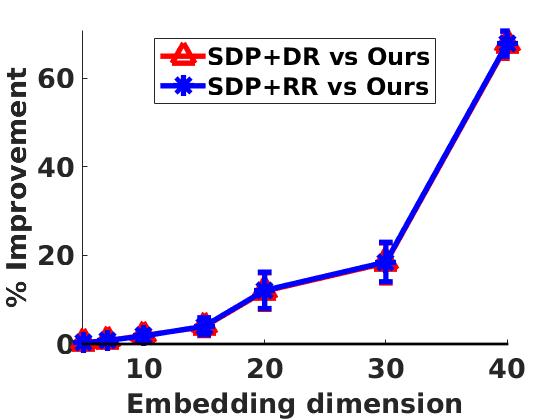}

\end{multicols}
\caption{\label{MNIST_Plots2} Performance of the proposed algorithm relative to baseline methods on MNIST and 20 Newsgroup datasets. Columns 1 and 2 correspond to MNIST dataset size of $50K$ and $100K$, respectively. The rows from top to bottom correspond to MNIST digit $2, 4, 5$, and $7$, respectively. The last column corresponds to the 20 Newsgroup dataset of size $1K$. Rows 1 and 2 of the last column correspond to categories Atheism and MS-Windows Misc, respectively.}
\end{figure*}

\begin{figure*} 
\begin{multicols}{2}
    \includegraphics[scale=0.3]{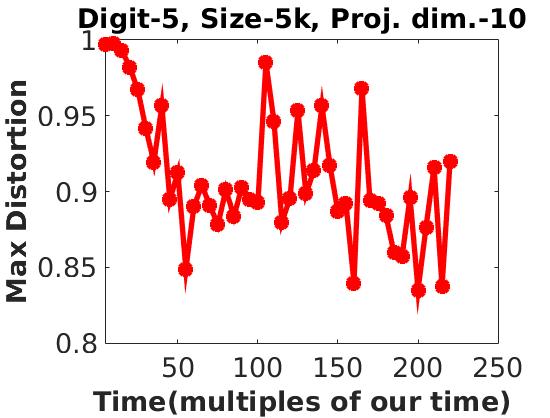}\par 
    \includegraphics[scale=0.3]{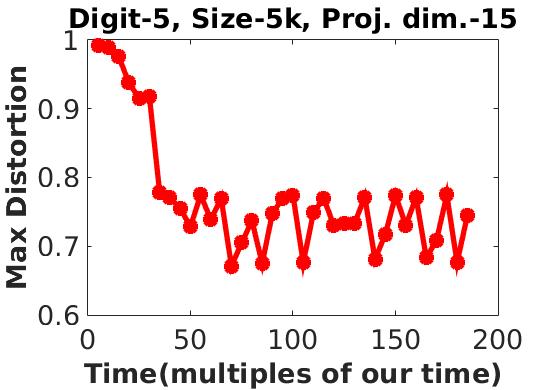}\par 
    \includegraphics[scale=0.3]{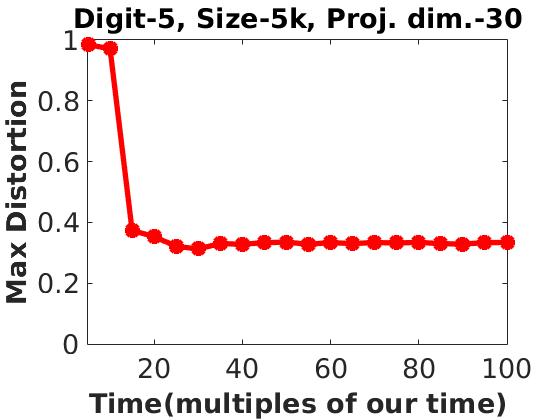}\par

\end{multicols}
\caption{\label{Indyk_Plots} Max distortion versus running time for SDP+DR method on MNIST digit 5 having dataset size of $5K$ and projection dimensions of $10, 15$ and $30$. Here running time is in multiples of the time taken by our algorithm.}
\end{figure*}

\begin{figure*} 
\begin{multicols}{2}
    \includegraphics[scale=0.3]{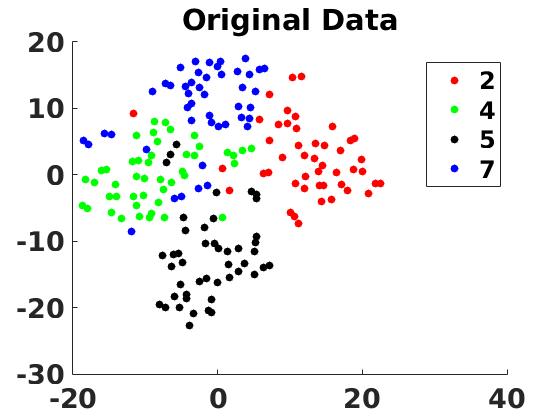}\par 
        \includegraphics[scale=0.3]{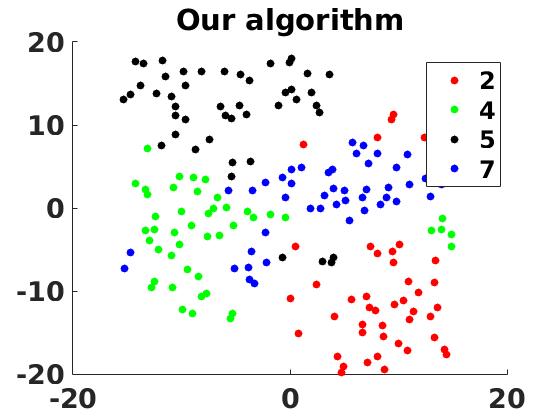}\par 
        \includegraphics[scale=0.3]{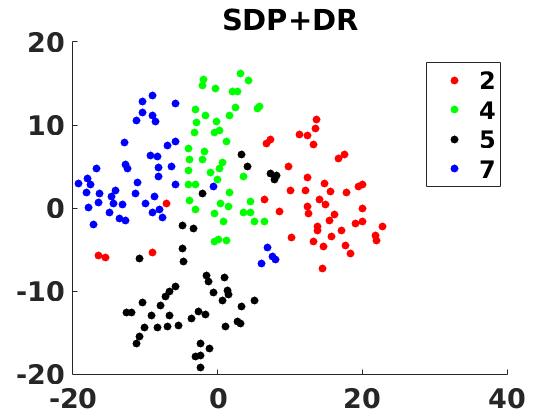}\par 
        \includegraphics[scale=0.3]{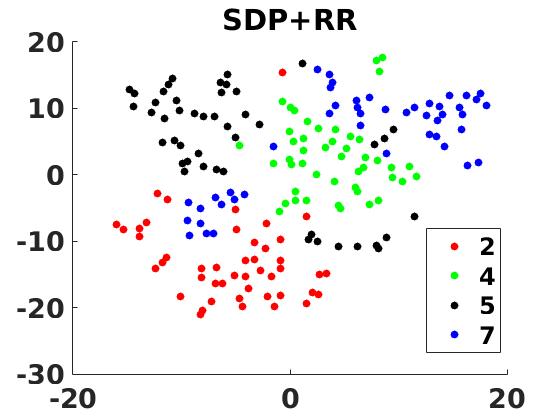}\par 
 
\includegraphics[scale=0.3]{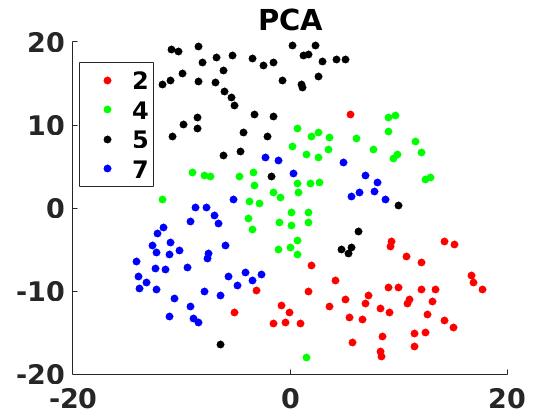}\par 
\end{multicols}
\caption{\label{t-SNE_Plots} t-SNE visualizations of the MNIST data in original space and 40 dimensional projected space using our algorithm as well as baseline methods.}
\end{figure*}

For MNIST dataset, we experimented with the images of the following digits - $2$, $4$, $5$, and $7$. For each of these four digits, we first randomly sampled certain number of images (that is, vectors of dimension $748$ each) from the dataset and then computed the pairwise normalized differences of these sampled images. This way, for each of these four digits, we obtained a collection of unit length vectors in $784$ dimensional space. Such a collection of these pairwise normalized images for each digit served as the input dataset for our experiments. 

In our experiments, for each of these four digits, we generated five datasets of different sizes - $1K, 5K, 10K, 50K$, and $100K$. Further, for each such dataset, we created $5$ variants by changing the seed of randomization used in picking the images. For each of these variants, we run our algorithm as well as baseline algorithms. The results of these experiments are summarized in the form of Figures \ref{MNIST_Plots1} and \ref{MNIST_Plots2}.

Observe that Figure \ref{MNIST_Plots1} consists of three columns -- each comprising four different plots. Each of these columns correspond the results for $1K, 5K$, and $10K$ sizes, respectively. The results for the sizes of $50K$, and $100K$ are being shown in columns 1 and 2, respectively, of Figure \ref{MNIST_Plots2}. The rows from top to bottom in each of the Figure \ref{MNIST_Plots1} and Figure \ref{MNIST_Plots2} (excluding the last column) correspond to the digits $2$, $4$, $5$, and $7$, respectively.   

Note that for each plot in Figure \ref{MNIST_Plots1} and Figure \ref{MNIST_Plots2} (excluding the last column), $x$-axis corresponds to the embedding dimension $(k)$ of the datasets which we have varied as $5, 7, 10, 15, 20, 30$, and $40$. The $y$-axis in each of these plots correspond to the quality of the solution (i.e., maximum distortion) offered by our algorithm as well as baseline algorithms. Only for $1K$ size in the first column of Figure \ref{MNIST_Plots1}, $y$-axis corresponds to the \%-improvement of our algorithm relative to  SDP+RR and SDP+DR. By \% improvement, we mean $100 \times \left(\text{\LARGE{$\epsilon$}}_{baseline} - \text{\LARGE{$\epsilon$}}_{ALG}\right)/\text{\LARGE{$\epsilon$}}_{baseline}$, where $\text{\LARGE{$\epsilon$}}_{baseline}$ is the maximum distortion offered by the baseline algorithm and $\text{\LARGE{$\epsilon$}}_{ALG}$ is the same quantity for our proposed algorithm. Note, higher the ratio means better is the proposed algorithm relative to baseline. The negative value of this ratio means our algorithm offers an inferior solution than the baseline algorithm. The reason for plotting \% improvement for the $1K$ size {\em but absolute score for all other sizes} is to have a better visual representation of the performance of our proposed algorithm relative to the baselines. For smaller dataset of size $1K$, we found that it becomes hard to visually compare the performance of our algorithm relative to the baselines if we opt for absolute score and this was the motivation for us to go for \% improvement in this case.

We would also like to highlight that each of the graph in these plots were plotted by taking average of $5$ variants of the same dataset size. The corresponding standard deviation bars are also being shown for each point in these plots.   

As far as time taken by our algorithm and baselines are concerned, recall that {Random} and {PCA} are not iterative algorithms and hence we always run them until we get their answer. On the other hand, our proposed algorithm and other two baselines - SDP+RR and SDP+DR are iterative in nature and hence it is important for us to fix an appropriate clock time restriction so as to make a fair comparison. In our experiments, we have put no time restriction on SDP+RR and SDP+DR methods in the case of $1K$ sized datasets and we allowed these baseline algorithms to run till their convergence. However, when we run our experiments on $5K$ and $10K$ sized datasets, we allowed these two baseline algorithms to have $3X$ more time than the time taken by our algorithm.  Similarly, for the case of $50K$ and $100K$ sized datasets, we offered the same time to the SDP+RR and SDP+DR methods as our algorithm. We selected some of these time restrictions  for the SDP+RR and SDP+DR methods because we observed that running these baselines until their stopping criteria was practically infeasible for the larger datasets. This observation itself underscores the scalability of our algorithm compared to these baselines. We have highlighted this observation in Section 4 already. 

It is evident from the plots in the Figures \ref{MNIST_Plots1} and \ref{MNIST_Plots2} that as far as MNIST dataset is concerned, our algorithm performs better than all the baseline methods almost all the times. Note that, we have not compared our results with {Random} and {PCA} baselines for the case of $1K$ dataset because both our algorithm as well as SDP+RR/SDP+RR perform far more superior than {Random} and {PCA} and hence we decide to compare our algorithm only with SDP+RR and SDP+DR for this specific case.

Finally, to make a compelling case in favor of our algorithm, we also conducted an additional experiment where simply run SDP+DR algorithm on $5K$ dataset until its convergence so as to ensure that its convergence time is significantly higher than the time offered by us to them during our experiments. Figure \ref{Indyk_Plots} depicts the results of this experiment. We performed this experiment only for the digit 5 and the projection dimensions of 10, 15, and 30.
From these plots, it is clear that SDP+DR algorithm (SDP+RR is anyways slower than SDP+DR) indeed takes significantly more times to converge compared to the time offered by us during our experiments. This rules out the possibility of the baseline algorithms converging quickly by giving any additional little time than what is offered by us in our experiments.
\end{paragraph}

\begin{paragraph}
{\bf 20 Newsgroup Dataset :} We repeated the same experimental setup for a different dataset, namely 20 Newsgroup. This dataset has much larger feature dimension (i.e. $8000$) as compared to MNIST. For this dataset, we picked two categories - {\em Atheism, MS-Windows Misc}. Our dataset size for each category was $1K$. This time, we offered $3X$ time to the SDR+DR and SDP+RR as compared to the time taken by our algorithm. In a manner similar to the plots for $1K$ sized MNIST datasets given in Figure \ref{MNIST_Plots1} (column 1), the plots in Figure \ref{MNIST_Plots2} (column 3) shows the percentage improvement of our algorithm compared to the baseline algorithms for the case of 20 Newsgroup dataset. We can observe from this experiment that our algorithm does not slow down with increasing the data dimension $d$. This is because we only need top-$k$ eigenvectors of the matrix $\textbf{M}\in \mathbb{R}^{d\times d}$ in each iteration.
\end{paragraph}

\begin{paragraph}
{\bf t-SNE plots :} 
Finally, to make a visual comparison between the quality of the embedding obtained by our algorithm and the baselines, we used t-SNE visualization algorithm. Figure \ref{t-SNE_Plots} depicts such visualization where we first visualize the MNIST data in the original $784$ dimensional space and then visualize its embedding in 40 dimensional space obtained via our algorithm,  SDP+DR, SDP+RR and PCA. We do so just for qualitative comparison purpose. From Figure \ref{t-SNE_Plots}, it appears that SDP+RR and SDP+DR distort the point clouds  more as compared to our algorithm.
\end{paragraph}


\section{Conclusion}
In this paper, we have presented a novel Lagrange duality based method to construct near isometric orthonormal linear embeddings. Our proposed algorithm reduces the dimension of the data while achieving lower distortion than the state-of-the-art baseline methods and is also computationally efficient. We have also given theoretical guarantees for the approximation quality offered by our algorithm. Our bound suggests that for certain input datasets, our algorithm offers near optimal solution of the problem. Our proposed theoretical guarantee depends on the spectral properties of the input data and hence the key question that we leave open is to obtain a data independent bound for the same. 
\section{Acknowledgement}
A part of this work for Dinesh Garg and Anirban Dasgupta was supported by the SERB-DST Early Career Research Grant No. ECR/2016/002035. The authors declare that they have no conflict of interest. 

Finally, the authors sincerely want to thank Dr. Chinmay Hegde of Iowa State University for sharing his code of SDP+RR and SDP+DR algorithms.
\bibliographystyle{spbasic}      
\bibliography{refs}   

\end{document}